\documentclass[10 pt]{article}

\usepackage{amsmath,amsthm,amssymb,amscd,amstext,amsfonts}
\usepackage{color,verbatim,graphicx,fullpage,url}
\usepackage[dvipsnames]{xcolor}
\IfFileExists{algorithm.sty}{\usepackage{algorithm}}{}
\IfFileExists{algorithmic.sty}{\usepackage{algorithmic}}{}
\usepackage{nicefrac}
\usepackage{fullpage,tikz,enumitem}
\usepackage{tikz-cd}
\usepackage{bm}
\IfFileExists{physics.sty}{\usepackage{physics}}{}
\providecommand{\norm}[1]{\lVert#1\rVert}

\usepackage{todonotes}
\usepackage{mathtools}
\usepackage{thm-restate}

\usepackage[breaklinks=true, linktocpage=true,hyperfootnotes=false]{hyperref}
\hypersetup{
    colorlinks = true,
    linkcolor={purple},
    urlcolor={purple},
    citecolor={blue!80!black}
}

\usepackage[nameinlink, noabbrev, capitalize]{cleveref}

\usepackage[utf8]{inputenc}
\usepackage[T1]{fontenc}

\Crefname{enumi}{Property}{Properties}

\theoremstyle{plain}

\newtheorem{thm}{Theorem}[section]
\newtheorem{theorem}[thm]{Theorem}

\crefname{atheorem}{Theorem}{Theorems}
\Crefname{atheorem}{Theorem}{Theorems}

\newtheorem{lemma}[thm]{Lemma}
\newtheorem{proposition}[thm]{Proposition}
\newtheorem{definition}[thm]{Definition}

\newtheorem*{claim*}{Claim}

\theoremstyle{remark}

\renewcommand\norm[1]{\left|\!\left|#1\right|\!\right|}

\newcommand{\Set}[1]{\left\{#1\right\}}

\def\1{\mathbf{1}} 
\def\0{\mathbf{0}}

\DeclareMathOperator{\Ex}{\mathbb{E}}           % expected value 
\def\supp{{\mathrm{supp}}}                               % support 

\newcommand{\R}{\mathbb{R}}

\DeclareMathOperator{\sign}{sgn}

\newcommand{\eps}{\varepsilon}

\usepackage{ifthen}
\newcommand{\agnorm}[2][]{
	\ifthenelse{\equal{#2}{}}{
		\widetilde{\gamma}_2^{#1}
	}{
		\widetilde{\gamma}_2^{#1}(#2)
	}
}

\IfFileExists{mathx10.tfm}{
\DeclareFontFamily{U}{mathx}{}
\DeclareFontShape{U}{mathx}{m}{n}{<-> mathx10}{}
\DeclareSymbolFont{mathx}{U}{mathx}{m}{n}
\DeclareMathAccent{\widecheck}{0}{mathx}{"71}

}{}

\newcommand{\cC}{\mathcal{C}}

\newcommand{\cP}{\mathcal{P}}

\newcommand{\E}{\mathbb{E}}

\DeclareMathOperator{\LR}{LR}

\DeclareMathOperator{\loss}{loss}                                             % loss

\title{A polylogarithmic bound for list replicability}

\author{
    Ari Blondal\thanks{McGill University, \texttt{ari.blondal@mail.mcgill.ca, hamed.hatami@mcgill.ca}. Hamed Hatami and Ari Blondal are supported by NSERC grants.} \and Hamed Hatami\footnotemark[1] \and
    Pooya Hatami~\thanks{Ohio State University, \texttt{\{hatami.2, lalov.1, tretiak.2\}@osu.edu}} \and  Chavdar Lalov\footnotemark[2] \and Sivan Tretiak\footnotemark[2]
}

\begin{document}

\maketitle
 
\begin{abstract}
Zhang~\cite{zhang_polylogarithmic_2026} recently posted a manuscript announcing that the list-replicability number of every class of $M\ge 2$ binary concepts on $N\ge 2$ domain points is $O(\log^2 M\log N)$, which would be the first polylogarithmic upper bound on this quantity. In this note, we present a short, self-contained proof of this bound, following the high-level approach of the manuscript but with considerably simpler arguments.
\end{abstract}

\section{Introduction}
In the study of PAC learning, \emph{replicability} is a guarantee that resampling training data will not dramatically alter the output of a learning algorithm. Such a property is of immense practical value for increasing model trustworthiness and reliability. Consequently, a variety of rigorous mathematical formulations of replicability have been introduced~\cite{ alon2022private, impagliazzo_reproducibility_2022, dixon2023listandcertificate, chase2023stability, hopkins2026approximatereplicabilitylearning}. Subsequent study has revealed nuanced implications between these formulations, sometimes depending on the cardinality of the domain underlying the learning task~\cite{moran_bayesian_2023, chase2023stability, bun_stability_2023, hopkins_role_2025}.

Of the assorted definitions of replicability, we will focus on \emph{list replicability}, introduced concurrently in~\cite{dixon2023listandcertificate, chase2023stability}. An algorithm learns a concept class \emph{list replicably} if (with high probability) it outputs a hypothesis in a small, accurate list. Although the list size is fixed, its members may depend on the underlying unknown distribution. Since the inception of list replicability, much progress has been made in understanding its properties and contextualizing it among more well-established learning theory invariants~\cite{ghazi2021sample, alon2022private, chase2023stability, hopkins_role_2025, blondal_sign-rank_2026}.

However, several fundamental questions about the definition remained open. Among these was the question of the optimal list size, known as the \emph{list replicability number}, in the finite setting. Consider a binary concept class $\cC \subseteq \{\pm 1\}^X$ with $|\cC|=M$ and $|X|=N$, where throughout the paper, we assume $M,N \ge 2$. Such a class can trivially be learned with a list comprised of all $M$ concepts, and it was shown \cite{chase2023stability} that a list of the same size $N$ as the domain also suffices. Together these give a rudimentary upper bound of $\min(M,N)$ on the list replicability number $\LR(\cC)$ of the class (see \Cref{def:lr} for a formal definition).

Until recently, it was not known whether the bound $\min(M,N)$ was close to optimal, or even whether a class exists that requires a list polynomially large in $M$ and $N$. We conjectured in~\cite{blondal_sign-rank_2026} that a list of size $O(\log M \log N)$ always suffices, and proved that this bound holds for random concept classes.\footnote{We write $\log$ for base $2$ logarithms and $\ln$ for natural logarithms.} This conjectured bound would be optimal up to constant factors (see~\cite{blondal_sign-rank_2026}). Recently, Zhang announced a proof \cite{zhang_polylogarithmic_2026} of an $O(\log^2 M \log N)$ upper bound, which would determine the list replicability number up to a factor of $\log M$.

\begin{theorem}[Main theorem]\label{thm:main}
There is an absolute constant $K>0$ such that every finite concept class $\cC\subseteq\{-1,1\}^X$ with $M=|\cC|\ge2$ and $N=|X|\ge2$ satisfies
\begin{equation}\label{eq:main-bound}
\LR(\cC)\le K\log^2 M\log N.
\end{equation}
\end{theorem}

The proof of \cref{thm:main} bounds only the size of the designated list; it provides no quantitative bounds on sample complexity, nor any guarantees of computational efficiency. 

One notable feature of the argument in \cite{zhang_polylogarithmic_2026} is that it aims to prove the upper bound on list size via a geometric viewpoint. This idea is not inherently new. Despite list replicability being a purely learning-theoretic definition, the optimal list size is closely tied to topological properties of the concept class \cite{chase_local_2024, blondal2026simplicial}. Indeed, the use of topological and geometric techniques in list replicability is a well-established line of work \cite{chase2023stability, dixon2023listandcertificate, chase_local_2024, blondal2026borsuk, blondal_sign-rank_2026, blondal2026simplicial, blondal_tight_2026, zhang_polylogarithmic_2026}. Broadly, these arguments involve analyzing the space of realizable distributions $\Delta_\cC$ associated to the concept class $\cC$ at hand. However, topological methods have seen the most success in proving impossibility results, that is, lower bounds on list size. The algorithm proposed in \cite{zhang_polylogarithmic_2026} is one of the rare list-replicable algorithms designed by working directly with the geometry of the space of realizable distributions $\Delta_\cC$. 

The polylogarithmic upper bound on the list size also has ramifications for topological invariants of the space $\Delta_\cC$ of realizable distributions. A $\mathbb{Z}_2$-space is a topological space with a continuous notion of antipodality (i.e., an involution with no fixed points). The $\mathbb{Z}_2$-index of such a space $Z$, denoted by $\operatorname{ind}_{\mathbb{Z}_2}(Z)$, is the smallest dimension $d$ such that there is a continuous antipodal map from $Z$ to the sphere $S^{d}$. The space $\Delta_{\cC \cup -\cC}$ of distributions realizable by concepts in $\cC$ or their negations is a $\mathbb{Z}_2$-space, with antipodality given by negating the labels; we define the \emph{$\mathbb{Z}_2$-index} of $\cC$ to be the $\mathbb{Z}_2$-index of this space. It is shown in~\cite{blondal_sign-rank_2026} that this topological invariant is bounded from above by twice the list replicability number of $\cC$:
\begin{equation}\label{eq:LR_to_ind}
\operatorname{ind}_{\mathbb{Z}_2}(\cC)\leq 2\LR(\cC)-1.
\end{equation}

Consequently, the polylogarithmic upper bound in \cite{zhang_polylogarithmic_2026} would immediately give an upper bound on the $\mathbb{Z}_2$-index. That bound is a weaker version of a concurrent result of \cite{frick_superpolynomial_2026}, which uses Morse-theoretic methods to prove a topological generalization of the Figiel--Lindenstrauss--Milman inequality to achieve an $O(\log M \log N)$ upper bound on $\mathbb{Z}_2$-index. Obtaining a polylogarithmic bound on the $\mathbb{Z}_2$-index was partly motivated by its consequence for triangulations of real projective space: it gives a superpolynomial lower bound on their number of vertices, which was unknown before~\cite{frick_superpolynomial_2026}.

This paper does not contain new results. Instead, it aims to greatly simplify the argument of \cite{zhang_polylogarithmic_2026}.

\paragraph{AI methodology.} 
The conjectured polylogarithmic bound of $O(\log M \log N)$ on list size was formulated in~\cite{blondal_sign-rank_2026} without AI input. After~\cite{zhang_polylogarithmic_2026} announced a 77 page proof of an $O(\log^2 M \log N)$ bound, the authors used AI to condense the manuscript into roughly 15 pages of notes, and then used these notes, together with follow-up questions to the AI, to extract its high-level approach.

The proofs and exposition in this note were drafted using a mix of human writing and AI tools. The proofs do not depend on the correctness of the arguments in Zhang's manuscript. After completing a working draft and verifying the proof by hand, the authors asked AI to check for bugs; it found nothing of note. AI was used to help with typesetting and spellcheck.

\paragraph{Paper organization.} We introduce and discuss the necessary learning-theoretic and topological preliminaries in \cref{sec:prelims}. We give an overview of the proof in \cref{sec:overview}.

The proof is first discussed for constant accuracy in \cref{sec:constant}, where the argument is divided into three steps: Gibbs sampling in \cref{sec:gibbs}, the fixed-distance net in \cref{sec:Sampling}, and dimension reduction in \cref{sec:compression}. Lastly, the result is extended to arbitrary accuracy in \cref{sec:amplification}.

\subsection{Preliminaries}\label{sec:prelims}

The initial setting for this paper is classical binary classification. Let $X$ be a \emph{finite domain} of size $N$, and take our labels to be in $\Set{-1,1}$. The learning task at hand is represented by a \emph{concept class} (also known as a hypothesis class) $\cC = \Set{c_1,\ldots,c_M} \subseteq \Set{-1,1}^X$ of size $M$. We make the additional non-degeneracy assumption that neither the domain nor the concept class consists of only one element, so $N=|X|\ge2$ and $M=|\cC|\ge2$.

Traditionally, a learning algorithm receives a sample drawn from a distribution $\mu$ on $X\times\Set{-1,1}$. The algorithm then outputs a \emph{hypothesis} $h\colon X\to\Set{-1,1}$. Note that we will allow \emph{improper learning}, where the output need not be in the concept class $\cC$. It will be convenient to further relax the setting to allow real-valued hypotheses $h\colon X\to\R$.

We define the \emph{loss} of $h$ with respect to the probability distribution $\mu$ on $X\times\Set{-1,1}$ by
\[
    \loss_\mu(h)\coloneqq\frac12\Ex_{(x,y)\sim\mu}|h(x)-y|.
\]
Note that rounding the real valued hypothesis $h$ to the Boolean hypothesis $x\mapsto\sign(h(x))$, with $\sign(0)=1$, increases its loss by at most a factor of two.

As is typical of papers studying list replicability from a topological viewpoint, we will almost completely eschew many of the algorithmic notions of learning theory going forward. The algorithmic definition of list replicability will feature only informally before introducing the closed cover characterization of list replicability that we will work with throughout the proof.

\paragraph{Topological method.}

Statistical learning generally assumes the distribution $\mu$ is \emph{realizable}, that is $\loss_\mu(c)=0$ for some $c\in\cC$. Algorithmically, list replicability requires that, for each realizable distribution, a learner's output belong with arbitrarily high probability to a list of at most $L$ hypotheses, each of loss at most $\eps$, when given sufficiently many independent examples. Crucially, the list is associated with the distribution rather than the observed sample.

The starting point for the topological method in list replicability is to geometrize the combinatorial information encoded by~$\cC$. Let $\Delta\subseteq\R^{X\times\Set{-1,1}}$ be the simplex of all probability distributions on $X\times\{-1,1\}$, and let $\Delta_\cC\subseteq\Delta$ be the space of realizable distributions, equipped with the $\ell_1$ metric. This space is compact, being a finite union of closed simplices.

We use the following topological characterization of list replicability as our definition. The \emph{multiplicity}
of a finite cover is the largest number of sets in the cover containing a
common point.

\begin{definition}[List replicability via closed covers]\label{def:lr}
For $\eps>0$, let $\LR(\cC,\eps)$ be the smallest positive integer $L$
for which there is a finite closed cover $\{R_a\}_{a\in A}$ of
$\Delta_\cC$ whose multiplicity is at most $L$, and hypotheses
$h_a$ for $a\in A$, such that
\[
    \mu\in R_a\quad\Longrightarrow\quad\loss_\mu(h_a)\le\eps.
\]
\end{definition}
The list replicability number is
\[
    \LR(\cC)\coloneqq\sup_{\eps>0}\LR(\cC,\eps).
\]

The closed-cover characterization of Chase, Moran, and
Yehudayoff~\cite{chase2023stability} relates our definition to
the algorithmic notion of list replicability; see also
\cite[Theorem~2.11]{blondal_sign-rank_2026}.
An $\eps$-accurate list learner with list size $L$, immediately yields a closed cover witnessing $\LR(\cC,\eps)\le L$. For the other direction, a closed cover witnessing $\LR(\cC,\eps/3)\le L$, can be rounded to a cover of the same multiplicity with Boolean hypotheses $h_a$ of loss $2\eps/3<\eps$. This allows enough slack, to turn this cover into an $\eps$-accurate list learner with list size at most $L$. 

The cover formulation lets us work entirely with the geometry of $\Delta_\cC$.

\subsection{Overview of the proof}\label{sec:overview} 
By \Cref{def:lr}, we need a finite closed cover $\{R_a\}$ of
$\Delta_\cC$ with small multiplicity such that each $R_a$ admits
a hypothesis $h_a$ with small loss for every distribution in $R_a$.

Since the topological dimension of $\mathbb{R}^k$ is $k$,  every compact subset of $\R^k$ admits arbitrarily fine finite
closed covers of multiplicity at most $k+1$. Therefore, it suffices to find a continuous map from $\Delta_\cC$ into $\R^k$, with
$k=O(\log^2 M\log N)$, such that the preimage of each sufficiently
small set admits a common accurate hypothesis. Pulling back a
fine cover of the image then gives the desired cover.

We describe the construction at constant loss $\eps=1/4$.
In \Cref{sec:amplification}, we extend it to arbitrary $\eps>0$
without increasing the multiplicity.
\begin{itemize}
\item \textbf{Distributions over concepts.}
We assign to each realizable distribution $\mu$ a probability
vector $q(\mu)=(q_1(\mu),\ldots,q_M(\mu))$ on the concepts whose average prediction
\[
h_{q(\mu)}(x)=\sum_{i\in[M]}q_i(\mu)c_i(x)
\]
has small loss. We require this assignment to be continuous, and its image to admit a small net. To make $h_{q(\mu)}$ have small loss, $q(\mu)$ should
put more weight on concepts $c_i$ of smaller loss. To make the net
small, however, $q(\mu)$ should not be too concentrated on a few
concepts; that is, it should have high entropy. We use Gibbs weights to balance these two requirements. 

Instead of working directly with $q$, we use the square-root representation
\[
u(\mu)=\bigl(\sqrt{q_i(\mu)}\bigr)_{i\in[M]},
\qquad U=u(\Delta_\cC)\subseteq\R^M,
\]
so that \emph{Euclidean distance} controls the differences between
the corresponding predictions:
\[
\norm{h_{q(\mu)}-h_{q(\mu')}}_\infty
\le \norm{q(\mu)-q(\mu')}_1
\le 2\norm{u(\mu)-u(\mu')}_2.
\]
\item \textbf{A small empirical net.}
Empirical approximation shows that $u(\mu)$ is within a
sufficiently small constant distance of $u(\widehat\mu)$ for
some empirical distribution $\widehat\mu$ on $s=O(\log^2 M)$
examples. The vectors $u(\widehat\mu)$ arising from realizable
samples therefore form a net $Z\subseteq U$, and counting the
possible samples gives
\[
\log|Z|\le O(s\log( N))=O(\log^2 M\log N).
\]
\item \textbf{Johnson--Lindenstrauss dimension reduction.}
The net $Z$ lies in $\R^M$, whose dimension is far larger than the
desired bound on $k$. However, the Johnson--Lindenstrauss lemma
provides an embedding of $Z$ into $\R^k$ with $k=O(\log|Z|)$ and
distortion arbitrarily close to $1$. Averaging the images of nearby
net points extends this embedding to a continuous map $F:U\to\R^k$
under which sets of sufficiently small diameter have preimages of
small diameter. This gives the maps
\[
\Delta_\cC \xrightarrow{\;u\;} U\subseteq\R^M \xrightarrow{\;F\;} F(U)\subseteq\R^k.
\]
Pulling back a fine cover of $F(U)$ along $F\circ u$ therefore gives
a cover of $\Delta_\cC$ with multiplicity at most $k+1$ and a common
predictor of small constant loss on each member. A similar
combination of a net and the Johnson--Lindenstrauss lemma for constructing such continuous maps was used in~\cite{hatami2023borsuk}.
\end{itemize}

\section{\texorpdfstring{\Cref{thm:main}}{Main theorem} at constant accuracy}\label{sec:constant}

In this section, we prove that $\LR(\cC,1/4)=O(\log^2 M\log N)$;
see \Cref{prop:constant-accuracy} below.

\subsection{Gibbs  sampling}\label{sec:gibbs}

Let $\cP_M$ denote the probability simplex on $[M]$. Each
$q\in\cP_M$ defines a randomized concept: on input $x$, sample
$i\sim q$ and output $c_i(x)$. Define
\[
h_q(x)\coloneqq \E_{i\sim q}c_i(x)
=\sum_{i\in[M]}q_i c_i(x).
\]

We seek a continuous map $\mu\in \Delta \mapsto q(\mu)\in\cP_M$ satisfying
two requirements: (i) the predictor $h_{q(\mu)}$ has small
$\mu$-loss for every realizable $\mu$, and (ii) the image of the map
admits a small net in a suitable Euclidean representation.

A natural choice is the Gibbs distribution associated with the
concept losses. For $\mu\in\Delta$ and $i\in[M]$, write
$\ell_i(\mu)\coloneqq\loss_\mu(c_i)$. Given a parameter $\beta>0$,
define
\[
q_i(\mu)\coloneqq\frac{e^{-\beta\ell_i(\mu)}}{Z_\mu},
\qquad
Z_\mu\coloneqq\sum_{j\in[M]}e^{-\beta\ell_j(\mu)}.
\]
These weights depend continuously on $\mu$ and favor concepts of smaller loss. Taking $\beta$ large concentrates the Gibbs weights $q_i$ on concepts of small loss and consequently improves the $\mu$-loss of $h_{q(\mu)}$. However, the trade-off is that it will lead to more samples in the empirical approximation used to construct our net. We will choose $\beta$
of order $\log M$ to obtain the desired constant accuracy while
keeping the net small.
\begin{lemma}\label{lem:gibbs}
For every realizable $\mu$,
\begin{equation}\label{eq:prediction-error}
\loss_\mu(h_{q(\mu)})\le\frac{\ln M}{\beta}.
\end{equation}
\end{lemma}

\begin{proof}
Since $y\in\{-1,1\}$, we have
\[
\loss_\mu(h_{q(\mu)})
=\frac12\E_{(x,y)\sim\mu}|h_{q(\mu)}(x)-y|
=\E_{(x,y)\sim\mu}\sum_i q_i(\mu)\1_{[c_i(x)\ne y]}
=\sum_i q_i(\mu)\ell_i(\mu).
\]
By the concavity of $\ln$,
\[
\beta\sum_i q_i(\mu)\ell_i(\mu)
=\sum_i q_i(\mu)\ln\!\left(e^{\beta\ell_i(\mu)}\right)
\le \ln\!\left(\sum_i q_i(\mu)e^{\beta\ell_i(\mu)}\right)=\ln\frac{M}{Z_\mu}
\le\ln M.
\]
The second equality uses
$q_i(\mu)e^{\beta\ell_i(\mu)}=1/Z_\mu$.
The last inequality follows from realizability: we have
$\ell_i(\mu)=0$ for some $i$, so $Z_\mu\ge1$.
\end{proof}

\subsection{A fixed-distance net}
\label{sec:Sampling}

The $\ell_1$ distance between weight vectors naturally controls
the difference between their predictions. To work with Euclidean
geometry, we use the square-root embedding
\[
u(\mu)\coloneqq
\left(\sqrt{q_1(\mu)},\ldots,\sqrt{q_M(\mu)}\right)
\in S^{M-1}.
\]

The image $U=u(\Delta_\cC)$ is a subset of the sphere $S^{M-1}$, and in
this section we show that it has a $\delta$-net of size
$(2N)^{\lceil\beta^2/\delta^2\rceil}$.

We next construct this net by
approximating the losses using a small sample.

\begin{proposition}\label{prop:net}
For every $0<\delta\le1$, the set $U$ has a $\delta$-net
$Z\subseteq U$ of size at most
$(2N)^{\lceil\beta^2/\delta^2\rceil}$.
\end{proposition}

\begin{proof}
Set $s=\lceil\beta^2/\delta^2\rceil$.
For a sample $(x_1,y_1),\ldots,(x_s,y_s)$, define
\[
\widehat\ell_i
=\frac1s\sum_{j=1}^s\1_{[c_i(x_j)\ne y_j]},
\qquad
\widehat q_i
=\frac{e^{-\beta\widehat\ell_i}}
       {\sum_k e^{-\beta\widehat\ell_k}},
\qquad
\widehat u=(\sqrt{\widehat q_i})_{i\in[M]}.
\]
Let $Z$ be the set of all vectors $\widehat u$ obtained as above, from realizable samples
of size $s$. There are at most $(2N)^s$ ordered samples, so $|Z|\le(2N)^s$.  Moreover, each $\widehat u\in Z$ lies in $U$: it equals $u(\nu)$, where $\nu$ is the distribution obtained by picking $j\in [s]$ uniformly at random and outputting $(x_j,y_j)$, which is realizable when $(x_1,y_1),\dots, (x_s,y_s)$ is consistent with some $c_i\in \cC$.

To show that $Z$ is a $\delta$-net, fix $\mu\in\Delta_\cC$ and draw
the sample independently from $\mu$. The sample is then consistent
with any concept realizing $\mu$, so $\widehat u\in Z$. It suffices to
show
\[
\E\norm{\widehat u-u(\mu)}_2^2\le\delta^2,
\]
since then some sample satisfies
$\norm{\widehat u-u(\mu)}_2\le\delta$.

Write $q_i=q_i(\mu)$ and put
\[
d_i=\widehat\ell_i-\ell_i(\mu),
\qquad
R=\frac{\sum_i e^{-\beta\widehat\ell_i}}{\sum_i e^{-\beta \ell_i}} = \sum_i q_i e^{-\beta d_i},
\]
so that $\widehat q_i=q_i e^{-\beta d_i}/R$.
Since $\widehat u$ and $u(\mu)$ are unit vectors,
\[
\norm{\widehat u-u(\mu)}_2^2
=2-2\sum_i\sqrt{q_i\widehat q_i}
=2\sum_i q_i\left(1-\sqrt{\widehat q_i/q_i}\right)
\le\sum_i q_i\ln\frac{q_i}{\widehat q_i}
=\beta\sum_i q_i d_i+\ln R,
\]
where the inequality is $1-t\le-\ln t$ with
$t=\sqrt{\widehat q_i/q_i}$.
Now take expectations. The weights $q_i$ do not depend on the sample
and each $d_i$ has mean zero, so the first term vanishes. Since
$d_i$ is an average of $s$ independent mean-zero variables in an
interval of length $1$, Hoeffding's lemma gives
$\E e^{-\beta d_i}\le e^{\beta^2/(8s)}$, and by concavity of $\ln$,
\[
\E\norm{\widehat u-u(\mu)}_2^2
\le\E\ln R
\le\ln\E R
\le\frac{\beta^2}{8s}
\le\delta^2.\qedhere
\]
\end{proof}

\subsection{Dimension reduction}\label{sec:compression}
So far, we have mapped $\Delta_\cC$ into $S^{M-1}$ and found a small
net $Z$ for the image $U=u(\Delta_\cC)$. While $S^{M-1}$ has large
dimension, the Johnson--Lindenstrauss lemma allows us to embed the net
$Z$ into $\R^k$ with $k=O(\log(|Z|))$ without distorting distances
by more than a constant factor. We then use the images of the net
points to define a continuous map from $U$ into $\R^k$, where the
desired covering bound follows from the fact that $\R^k$ has
covering dimension $k$.

\begin{lemma}\label{lem:lebesgue}
Let $Y\subseteq\R^k$ be compact and $\eta>0$. There is a finite closed
cover of $Y$ whose members have diameter at most $\eta$ and whose
multiplicity is at most $k+1$.
\end{lemma}
\begin{proof}
The covering dimension of $\R^k$ is $k$, so the compact subspace $Y$
has covering dimension at most $k$. Cover $Y$ by finitely many open
sets of diameter at most $\eta$. The defining refinement property of
covering dimension gives a finite open refinement of multiplicity at
most $k+1$. A finite open cover of a compact metric space has a closed
shrinking: each member can be replaced by a closed subset so that the
subsets still cover the space. This preserves the diameter bound and
does not increase multiplicity.
\end{proof}

\begin{lemma}\label{lem:net-cover}
Let $U\subseteq S^{M-1}$ be compact, let $\delta>0$, and suppose $U$ has a finite
$\delta$-net $Z\subseteq U$. Then $U$ has a finite closed cover whose
members have diameter at most $14\delta$ and whose multiplicity is
$O(\log(2|Z|))$. 
\end{lemma}

\begin{proof}
By the Johnson--Lindenstrauss lemma~\cite{JL1984,DG2003}, there is a linear map
$T:\R^M\to\R^k$ with $k=O(1+\log(|Z|))$ satisfying
\[
\frac12\norm{z-z'}_2^2
\le\norm{Tz-Tz'}_2^2
\le\frac32\norm{z-z'}_2^2
\qquad(z,z'\in Z).
\]
For $w\in U$, define
\[
\lambda_z(w)=
\frac{(2\delta-\norm{w-z}_2)_+}
{\sum_{a\in Z}(2\delta-\norm{w-a}_2)_+},
\qquad
F(w)=\sum_{z\in Z}\lambda_z(w)Tz,
\]
where $t_+=\max\{t,0\}$. The denominator is positive because $Z$ is a
$\delta$-net. Thus $F$ is continuous and is a convex combination of the
images of nearby net points.

Choose $z\in Z$ with $\norm{w-z}_2\le\delta$. Whenever
$\lambda_a(w)>0$, we have
\[
\norm{a-z}_2\le\norm{a-w}_2+\norm{w-z}_2<3\delta.
\]
Consequently,
\[
\norm{F(w)-Tz}_2
\le\sum_{a\in Z}\lambda_a(w)\norm{Ta-Tz}_2
\le3\sqrt{\tfrac32}\,\delta.
\]
For $w,w'\in U$, choose corresponding net points $z,z'$ as above. Then
\begin{align}
\norm{w-w'}_2
&\le2\delta+\norm{z-z'}_2\notag\\
&\le2\delta+\sqrt2\norm{Tz-Tz'}_2\notag\\
&\le(2+6\sqrt3)\delta+\sqrt2\norm{F(w)-F(w')}_2.\label{eq:compression}
\end{align}

The set $F(U)\subseteq\R^k$ is compact. Apply \Cref{lem:lebesgue} to it
with $\eta=\delta$, and denote the resulting cover by $\{D_a\}_{a\in A}$.
The sets $F^{-1}(D_a)$ form a finite closed cover of $U$ of multiplicity
at most $k+1$, and by~\eqref{eq:compression} each has diameter at most
$(2+6\sqrt3+\sqrt2)\delta<14\delta$.
\end{proof}

We can now prove the bound at accuracy $1/4$. 

\begin{proposition}\label{prop:constant-accuracy}
Set $\beta=32\ln M$ and $U=u(\Delta_\cC)$. Then $U$ has a finite
closed cover $\{D_a\}_{a\in A}$ whose members have diameter at most $1/16$ and
multiplicity $L=O(\log^2 M\log N)$. Consequently,
\[
\LR(\cC,1/4)\le L.
\]
\end{proposition}

\begin{proof}
Apply \Cref{prop:net} with $\delta=1/256$ to obtain a net $Z$ for $U$
with $\log(|Z|)=O(\beta^2\log(N))=O(\log^2 M\log N)$.
By \Cref{lem:net-cover}, $U$ has a finite closed cover  $\{D_a\}_{a\in A}$ of multiplicity $L=O(\log(2|Z|))$ whose members
have diameter at most $14/256<1/16$. The sets
\[
    R_a\coloneqq\{\mu\in\Delta_\cC:u(\mu)\in D_a\}
\]
form a finite closed cover of $\Delta_\cC$ of multiplicity at most $L$.

For each nonempty $R_a$, choose $\nu_a\in\Delta_\cC$ such that
$u(\nu_a)\in D_a$, and assign the hypothesis
$h_a\coloneqq h_{q(\nu_a)}$.
For $\mu\in R_a$,  
Cauchy--Schwarz  gives
\begin{align*}
\norm{h_{q(\mu)}-h_{q(\nu_a)}}_\infty
&\le \sum_i|q_i(\mu)-q_i(\nu_a)|
=\sum_i|u_i(\mu)^2-u_i(\nu_a)^2|\\
&\le \norm{u(\mu)-u(\nu_a)}_2\norm{u(\mu)+u(\nu_a)}_2
\le 2\norm{u(\mu)-u(\nu_a)}_2\le\frac18.
\end{align*}
Therefore, by the triangle inequality and \Cref{lem:gibbs},
\begin{align*}
\loss_\mu(h_a)
&\le\loss_\mu(h_{q(\mu)})
+\norm{h_{q(\mu)}-h_{q(\nu_a)}}_\infty\\
&\le\frac{\ln M}{\beta}+\frac18
=\frac{5}{32}\le\frac14. \qedhere
\end{align*}

\end{proof}

\section{Proof of \Cref{thm:main}}\label{sec:amplification}

To improve the accuracy, we choose a reweighting $\nu$ of $\mu$ that
minimizes the Gibbs partition function
$\sum_{i}e^{-\beta\ell_i(\nu)}$, subject to
$\nu(z)\le\mu(z)/\eps$. This constraint allows us to condition $\mu$
on any set of measure at least $\eps$, where the optimality condition
will imply that the resulting Gibbs predictor has small loss
under every such conditioning.

\begin{lemma}\label{lem:reweighting}
Fix $\beta>4\ln M$ and $0<\eps<1$. There is a continuous map
$p_\eps : \Delta_\cC \to q(\Delta_\cC)$, $\mu \mapsto p_\mu$ such that
\begin{equation}\label{eq:amplified-tail}
\mu\Set{(x,y):|h_{p_\mu}(x)-y|\ge1/2}<\eps.
\end{equation}
\end{lemma}

\begin{proof}
For $\mu\in\Delta_\cC$, let
\[
K_\mu\coloneqq
\Set{\nu\in\Delta:\nu(z)\le\mu(z)/\eps
\text{ for every }z\in X\times\{-1,1\}}.
\]
This is a nonempty compact convex set, since it contains $\mu$.
Every $\nu\in K_\mu$ is supported on $\supp(\mu)$ and is therefore
realizable. Choose $\nu\in K_\mu$ minimizing
\[
Z_\nu=\sum_{i\in[M]}e^{-\beta\ell_i(\nu)},
\]
and define $p_\mu \coloneqq q(\nu)$. Although $\nu$ need not be unique, its loss
vector $(\ell_i(\nu))_{i\in[M]}$ is unique: otherwise, strict convexity of
$\sum_i e^{-\beta t_i}$ would give a smaller value at the midpoint
of two minimizing loss vectors. Thus $p_\mu$ is well defined.

It is straightforward to check that $\mu \mapsto p_\mu$ is continuous.   It remains to verify \eqref{eq:amplified-tail}. Write $p=p_\mu$ and
\[
L_p(x,y)\coloneqq\frac12|h_p(x)-y|
=\sum_i p_i\1_{[c_i(x)\ne y]}.
\]
For every $\rho\in K_\mu$, minimality of $\nu$ gives
\[
0\le
\left.\frac{d}{dt}Z_{(1-t)\nu+t\rho}\right|_{t=0}
=-\beta Z_\nu\bigl(\E_\rho L_p-\E_\nu L_p\bigr).
\]
Consequently, by \Cref{lem:gibbs},
\[
\E_\rho L_p\le\E_\nu L_p\le\frac{\ln M}{\beta}<\frac14.
\]
If $A=\{z:L_p(z)\ge1/4\}$ had $\mu(A)\ge\eps$, then the
conditional distribution $\rho=\mu(\,\cdot\mid A)$ would belong
to $K_\mu$ and satisfy $\E_\rho L_p\ge1/4$, a contradiction.
This proves \eqref{eq:amplified-tail}.
\end{proof}

\begin{proof}[Proof of \Cref{thm:main}]
Set $\beta=32\ln M$. \Cref{prop:constant-accuracy}
provides a finite closed cover $\{D_a\}_{a\in A}$ of
$U=u(\Delta_\cC)$ with diameter at most $1/16$ and multiplicity
$L=O(\log^2 M\log N)$. Discard empty members, choose
$\rho_a$ with $u(\rho_a)\in D_a$, and set
\[
h_a(x)\coloneqq\sign(h_{q(\rho_a)}(x)).
\]

Fix $0<\eps<1$. By \Cref{lem:reweighting}, the map
\[
v_\eps(\mu)\coloneqq(\sqrt{p_{\mu,i}})_{i\in[M]}
\]
is continuous and takes values in $U$. Thus
$R_a^\eps=v_\eps^{-1}(D_a)$ form a finite closed cover of
$\Delta_\cC$ of multiplicity at most $L$.

For $\mu\in R_a^\eps$, Cauchy--Schwarz gives
\[
\norm{h_{p_\mu}-h_{q(\rho_a)}}_\infty
\le2\norm{v_\eps(\mu)-u(\rho_a)}_2\le\frac18.
\]
Hence $h_a(x)\ne y$ implies $|h_{p_\mu}(x)-y|\ge7/8$, so
\Cref{lem:reweighting} yields
\[
\loss_\mu(h_a)
\le\mu\Set{(x,y):|h_{p_\mu}(x)-y|\ge1/2}\le\eps.
\]
Since $\eps$ was arbitrary,
$\LR(\cC)\le L=O(\log^2 M\log N)$.
\end{proof}

\bibliographystyle{alpha}
\bibliography{refs}

\end{document}